\documentclass{article}

\usepackage{graphicx} 
\usepackage{subfigure} 

\usepackage{natbib}

\usepackage{algorithm}
\usepackage{algorithmic}

\usepackage{hyperref}

\usepackage[accepted]{icml2012} 
\usepackage{amssymb,amsmath,amsthm}

\icmltitlerunning{Complexity Analysis of the Lasso Regularization Path}

\def\x{{\mathbf x}}
\def\p{{\mathbf p}}

\def\1{{\mathbf 1}}

\def\X{{\mathbf X}}

\def\etab{{\boldsymbol\eta}}
\def\tildeetab{{\boldsymbol{\tilde{\eta}}}}

\def\y{{\mathbf y}}
\def\w{{\mathbf w}}

\def\u{{\mathbf u}}

\def\tildeX{{\bf \tilde X}}

\def\tildey{{\bf \tilde y}}

\def\tildew{{\bf \tilde w}}
\def\Real{{\mathbb R}}

\def\u{{\mathbf u}}

\def\argmin{\operatornamewithlimits{arg\,min}}

\def\sign{\operatorname{sign}}
\def\st{~~\text{s.t.}~~}
\def\defin{\triangleq}

 \newcommand{\INPUT}{ \STATE {\textbf{Inputs:}} }

\def \kappab {{\boldsymbol\kappa}}

\long\def\symbolfootnote[#1]#2{\begingroup\def\thefootnote{\fnsymbol{footnote}}\footnote[#1]{#2}\endgroup} 

\newtheorem{theorem}{Theorem}
\newtheorem{lemma}{Lemma}
\newtheorem{proposition}{Proposition}
 \newtheorem{definition}{Definition}

\begin{document} 

\twocolumn[
\icmltitle{Complexity Analysis of the Lasso Regularization Path}

\icmlauthor{Julien Mairal}{julien@stat.berkeley.edu}
\icmlauthor{Bin Yu}{binyu@stat.berkeley.edu}
\icmladdress{Department of Statistics, University of California, Berkeley.}

\icmlkeywords{Lasso, regularization path, homotopy}

\vskip 0.2in  
]

\begin{abstract} 
The regularization path of the Lasso can be shown to be piecewise linear,
making it possible to ``follow'' and explicitly compute the entire path. We
analyze in this paper this popular strategy, and prove that its worst case
complexity is exponential in the number of variables.  We then oppose this
pessimistic result to an (optimistic) approximate analysis: We show that an
approximate path with at most $O(1/\sqrt{\varepsilon})$ linear segments can
always be obtained, where every point on the path is guaranteed to be optimal
up to a relative $\varepsilon$-duality gap. We complete our theoretical
analysis with a practical algorithm to compute these approximate paths.

\end{abstract} 

\vspace*{-0.5cm}
\section{Introduction}
\vspace*{-0.08cm}
Without a priori knowledge about data, it is often difficult to estimate
a model or make predictions, either because the number
of observations is too small, or the problem dimension too high. When a problem
solution is known to be sparse, sparsity-inducing penalties have proven to be
useful to improve both the quality of the prediction and its intepretability.
In particular, the~$\ell_1$-norm has been used for that purpose in the Lasso formulation~\citep{tibshirani}.

Controlling the regularization often requires to tune a
parameter.  In a few cases, the regularization path---that is, the set of
solutions for all values of the regularization parameter, can be shown to be
piecewise linear~\citep[][]{rosset}. This property is exploited in
homotopy methods, which consist of following the piecewise linear path by
computing the direction of the current linear segment and the points where the
direction changes (also known as kinks).  Piecewise linearity of regularization
paths was discovered by~\citet{markowitz} for portfolio selection; it was
similarly exploited by~\citet{osborne2} and~\citet{efron} for the Lasso, and
by~\citet{hastie} for the support vector machine (SVM).  As observed
by~\citet{gartner}, all of these examples are in fact particular instances of
\emph{parametric quadratic programming} formulations, for which path-following
algorithms appear early in the optimization literature~\citep{ritter}.

In this paper, we study the number of linear segments of the Lasso
regularization path. Even though experience with data suggests that this
number is linear in the problem size~\citep{rosset}, it is known that discrepancies can be observed
between worst-case and empirical complexities. This is notably the case for the
simplex algorithm~\citep{dantzig}, which performs empirically well for
solving linear programs even though it suffers from exponential worst-case
complexity~\citep{klee}.  Similarly, by using geometrical
tools originally developed to analyze the simplex algorithm,~\citet{gartner} have shown that the
complexity of the SVM regularization path can be exponential.  However, to the
best of our knowledge, none of these results do apply to the Lasso regularization path,
whose theoretical complexity remains unknown. The goal of our paper is to fill
in this~gap. 

Our first contribution is to show that in the worst-case the number of linear
segments of the Lasso regularization path is exactly~$(3^p\!+\!1)/2$, where~$p$ is
the number of variables (predictors). We remark that our proof is constructive and significantly
different than the ones proposed by~\citet{klee} for the simplex algorithm and
by~\citet{gartner} for SVMs. Our approach does not rely on geometry but on an
adversarial scheme. Given a Lasso problem with~$p$ variables, we show how to
build a new problem with~$p+1$ variables increasing the complexity of the path
by a multiplicative factor.  It results in explicit pathological examples that are
surprisingly simple, unlike pathological examples for the simplex algorithm
or~SVMs.

Worst-case complexity analyses are by nature pessimistic. Our second
contribution on approximate regularization paths is more optimistic.
In fact, we show that an approximate path
for the Lasso with at most $O(1/\sqrt{\varepsilon})$ segments can always be
obtained, where every point on the path is guaranteed to be optimal up to a
relative $\varepsilon$-duality gap. We follow in part the methodology
of~\citet{giesen2} and~\citet{giesen}, who have presented weaker results but in a more general setting for
parameterized convex optimization problems. Our analysis builds upon approximate
optimality conditions, which we maintain along the path, leading to a practical approximate homotopy algorithm.

The paper is organized as follows: Section~\ref{sec:related} presents some
brief overview of the Lasso.  Section~\ref{sec:worst} is devoted to our worst-case
complexity analysis, and Section~\ref{sec:approx} to our results on approximate
regularization paths.

\section{Background on the Lasso}\label{sec:related}
In this section, we present the Lasso formulation of~\citet{tibshirani}
and well known facts, which we exploit later in our analysis.
For self-containedness and clarity reasons we include simple proofs of these results.
Let~$\y$ be a vector in~$\Real^n$ and~$\X=[\x^1,\ldots,\x^p]$ be a matrix in~$\Real^{n \times p}$. The Lasso is formulated as:
\begin{equation}
   \min_{\w \in \Real^p} \frac{1}{2}\|\y-\X\w\|_2^2 + \lambda \|\w\|_1,\label{eq:lasso}
\end{equation}
where the~$\ell_1$-norm induces sparsity in the solution~$\w$
and~$\lambda\!>\!0$ controls the amount of regularization.
Under a few assumptions, which are detailed in the sequel,
the solution of this problem is
unique. We denote it by~$\w^\star(\lambda)$ and
define the \emph{regularization path} ${\mathcal P}$ as the set of all solutions for all
positive values of~$\lambda$:\footnote{For technicality reasons, we enforce~$\lambda\!>\!0$ even though the limit $\w^\star(0^+) \defin \lim_{\lambda
\to 0^+} \w^\star(\lambda)$ may exist.}
 $${\mathcal P} \defin \{ \w^\star(\lambda) : \lambda > 0 \}.$$
 The following lemma presents classical optimality and uniqueness conditions for the Lasso solution~\citep[see][]{fuchs}, which are useful to characterize~$\mathcal P$:
\begin{lemma}[{\bfseries Optimality Conditions of the Lasso}]\label{lemma:opt}
A vector~$\w^\star$ in~$\Real^p$ is a solution of Eq.~(\ref{eq:lasso}) if and only if for all~$j$ in $\{1,\ldots,p\}$, 
\begin{equation}
   \begin{split}
   \x^{j\top}(\y-\X\w^\star) &=\lambda\sign(\w^\star_j) ~~\text{if}~~ \w^\star_j \neq 0, \\
   |\x^{j\top}(\y-\X\w^\star)|  &\leq\lambda ~~\text{otherwise}.
   \end{split} \label{eq:opt}
\end{equation}
Define $J \defin \{ j \in \{1,\ldots,p\} : |\x^{j\top}(\y-\X\w^\star)|=\lambda \}$.
Assuming the matrix~$\X_J=[\x^j]_{j\in J}$ to be full rank, the solution is unique and we have  
\begin{equation}
\w_J^\star =
(\X_J^\top\X_J)^{-1}(\X_J^\top\y-\lambda \etab_J), \label{eq:closed}
\end{equation}
where $\etab \defin \sign(\X^\top(\y-\X\w^\star))$ is in $\{-1; 0; +1\}^p$, and
the notation~$\u_J$ for a vector~$\u$ denotes the vector of size~$|J|$
recording the entries of~$\u$ indexed by~$J$.
\end{lemma}
\begin{proof}
   Eq.~(\ref{eq:opt}) can be obtained by considering subgradient
   optimality conditions.  These can be written as
    $  0 \in \{ -\X^\top(\y-\X\w^\star) +
      \lambda \p  :  \p  \in \partial \|\w^\star\|_1 \}$,
   where $\partial\|\w^\star\|_1$ denotes the subdifferential of the $\ell_1$-norm at~$\w^\star$. A classical result~\citep{borwein}
   says that the subgradients $\p$ are the vectors in~$\Real^p$ such that for all~$j$ in
   $\{1,\ldots,p\}$, $\p_j=\sign(\w^\star_j)$ if $\w^\star_j \neq
   0$, and $|\p_j| \leq 1$ otherwise.  This gives Eq.~(\ref{eq:opt}).
   The equalities in Eq.~(\ref{eq:opt}) define a linear system that has a
   unique solution given by~(\ref{eq:closed}) when~$\X_J$ is full~rank. 

   Let us now show the uniqueness of the Lasso solution. Consider another solution $\w^{\prime\star}$ and choose a scalar~$\theta$ in $(0,1)$.
By convexity, $\w^{\theta \star} \defin \theta \w^\star +
(1-\theta)\w^{\prime\star}$ is also a solution. 
For all $j \notin J$, we have $|\x^{j\top}(\y-\X\w^{\theta\star})| \leq \theta|\x^{j\top}(\y-\X\w^{\star})| + (1-\theta) |\x^{j\top}(\y-\X\w^{\prime\star})| <  \lambda$.
Combining this inequality with the conditions~(\ref{eq:opt}), we necessarily have $\w_{J^\complement}^{\theta\star}=\w_{J^\complement}^{\star}=0$,\footnote{$J^\complement$ denotes the complement of the set~$J$ in~$\{1,\ldots,p\}$.}
and the vector~$\w^{\theta\star}_J$ is also a solution of the following reduced
problem:
\begin{displaymath}
\vspace*{-0.1cm}
   \min_{\tildew \in \Real^{|J|}} \frac{1}{2}\|\y-\X_J\tildew\|_2^2 + \lambda \|\tildew\|_1. 
\vspace*{-0.1cm}
\end{displaymath}
When~$\X_J$ is full rank, the Hessian~$\X_J^\top\X_J$ is
positive definite and this reduced problem is strictly convex. Thus, it admits a unique
solution $\w^{\theta\star}_J=\w^\star_J$. It is then easy to conclude that
$\w^\star=\w^{\theta\star}=\w^{\prime\star}$. 
\end{proof}
\vspace*{-0.1cm}
With the assumption that the matrix~$\X_J$ is always full-rank, we can formally recall a well-known
property of the Lasso~\citep[see][]{markowitz,osborne2,efron} in the following lemma:
\begin{lemma}[{\bfseries Piecewise Linearity of the Path}]~\label{lemma:uniqpath}\newline
Assume that for any~$\lambda >0$ and solution of Eq.~(\ref{eq:lasso}) the matrix~$\X_J$ defined in Lemma~\ref{lemma:opt} is full-rank. Then, the regularization path~$\{ \w^\star(\lambda) : \lambda > 0\}$ is
well defined, unique and continuous piecewise linear.
\end{lemma}
\vspace*{-0.2cm}
\begin{proof}
   The existence/uniqueness of the regularization path was shown in Lemma~\ref{lemma:opt}.
   
Let us define~$\{ \etab^\star(\lambda)\defin
\sign(\w^\star(\lambda)) : \lambda > 0\}$ the set of sparsity patterns.
   Let us now consider $\lambda_1 < \lambda_2$ such that
$\etab^\star(\lambda_1)=\etab^\star(\lambda_2)$.
   For all~$\theta \in [0,1]$, it is easy to see that the
solution~$\w^{\theta\star} \defin \theta\w^\star(\lambda_1) +
(1-\theta)\w^\star(\lambda_2)$ satisfies the optimality conditions of
Lemma~\ref{lemma:opt} for~$\lambda=\theta\lambda_1\!+\!(1\!-\!\theta)\lambda_2$, and
that
$\w^\star(\theta\lambda_1+(1\!-\!\theta)\lambda_2)\!=\w^{\theta\star}$.

This shows that whenever two solutions~$\w^\star(\lambda_1)$
and~$\w^\star(\lambda_2)$ have the same signs for $\lambda_1\!\neq\!\lambda_2$,
the regularization path between~$\lambda_1$ and~$\lambda_2$ is a linear
segment.
As an important consequence, the number of
linear segments of the path is smaller than~$3^p$,
the number of possible sparsity patterns in~$\{-1,0,1\}^p$. The path~$\mathcal P$ is therefore
piecewise linear with a finite number of~kinks. 

Moreover, since the function $\lambda \to \w^\star(\lambda)$ is piecewise
linear, it is piecewise continuous and has right and left limits for
every~$\lambda\!>\!0$. It is easy to show that these limits satisfy 
the optimality conditions of Eq.~(\ref{eq:opt}). By uniqueness of the Lasso
solution, they are equal to~$\w^\star(\lambda)$ and the function is in fact
continuous.
\end{proof}
Assuming again that~$\X_J$ is always full rank, we can now present in Algorithm~\ref{alg:homotopy} the homotopy
method~\citep{osborne2,efron}.

\vspace*{-0.1cm}
\begin{algorithm}[hbtp]
\caption{Homotopy Algorithm for the Lasso.}\label{alg:homotopy}
\begin{algorithmic}[1]
\STATE {\bfseries Inputs:} a vector~$\y$ in $\Real^n$; a matrix~$\X$ in~$\Real^{n \times p}$;
\STATE {\bfseries initialization:} set $\lambda$ to $\|\X^\top\y\|_\infty$; we have $\w^\star(\lambda)=0$ (trivial solution);
\STATE set $J\defin \{ j_0 \}$ such that~$|\x^{j_0\top}\y|
    =\lambda$;
     \WHILE{$\lambda > 0$} 
     \STATE Set~$\etab\defin\sign(\X^\top(\y\!-\!\X\w^\star(\lambda))$;
     \STATE
     \label{item:lars} compute the direction of the path: 
     \vspace*{-0.2cm}
     \begin{displaymath} 
     \left\{ \begin{array}{rcl}
     \w_J^\star(\lambda) & \!=\! &(\X_J^\top\X_J)^{-1}(\X_J^\top\y\!-\!\lambda \etab_J) \\
     \w^\star_{J^\complement}(\lambda)&\!=\!&0. 
     \end{array} \right.
     \vspace*{-0.35cm}
     \end{displaymath}
     \STATE Find the smallest step~$\tau > 0$ such that: \label{step:tau} \\
         $\bullet$~there exists $j \in J^\complement$ such that \\ $|\x^{j\top}(\y\!-\!\X\w^\star(\lambda\!-\!\tau))| =\lambda\!-\!\tau$; add $j$ to~$J$; \\
         $\bullet$~there exists $j$ in $J$ such that~$\w^\star_j(\lambda) \!\neq \!0$ and $\w^\star_j(\lambda\!-\!\tau)=0$; remove $j$ from $J$;
     \STATE replace~$\lambda$ by~$\lambda-\tau$; record the pair $(\lambda,\w^\star(\lambda))$; 
     \ENDWHILE
\STATE {\bf{Return:}} sequence of recorded values~$(\lambda,\w^\star(\lambda))$.
\end{algorithmic}
\end{algorithm}
\vspace*{-0.05cm}

It can be shown that this algorithm maintains the optimality conditions of
Lemma~\ref{lemma:opt} when~$\lambda$ decreases. Two assumptions have
nevertheless to be made for the algorithm to be correct.  First, $(\X_J^T
\X_J)$ has to be invertible, which is a reasonable assumption commonly made when
working with real data and when one is interested in sparse solutions. When
$(\X_J^T \X_J)$ becomes ill-conditioned, which may typically occur for small values of $\lambda$,
the algorithm has to stop and the path is truncated.
Second, one assumes in Step~\ref{step:tau} of the algorithm that the value~$\tau$ corresponds to a 
single event $|\x^{j\top}(\y\!-\!\X\w^\star(\lambda\!-\!\tau))| =\lambda\!-\!\tau$ for~$j$ in~$J^\complement$
or $\w^\star_j(\lambda\!-\!\tau)$ hits zero for~$j$ in~$J$.
In other words, variables enter or exit the path one at a time.
Even though this assumption is reasonable most of the time, it can be
problematic from a numerical point of view in rare cases.  When the length of a
linear segment of~$\mathcal P$ is smaller than the numerical precision, the
algorithm can fail.  In contrast, our approximate homotopy algorithm presented
in Section~\ref{sec:approx} is robust to this issue.  In the next section, we
present our worst-case complexity analysis of the regularization path, showing
that Algorithm~\ref{alg:homotopy} can have exponential complexity.

\section{Worst-Case Complexity}\label{sec:worst}
We denote by~$\{ \etab^\star(\lambda) \defin \sign(\w^\star(\lambda)) : \lambda
> 0 \}$  the set of sparsity patterns in~$\{-1,0,1\}^p$ encountered along the path~$\mathcal P$. We have seen in the proof
of Lemma~\ref{lemma:uniqpath} that whenever
$\etab^\star(\lambda_1)=\etab^\star(\lambda_2)$, for~$\lambda_1,\lambda_2 > 0$,
then $\etab^\star(\lambda)=\etab^\star(\lambda_1)$ for all~$\lambda \in
[\lambda_1,\lambda_2]$, and thus the number of linear segments of~$\mathcal P$
is upper-bounded by~$3^p$.
With an additional argument, we can further reduce this number, as stated in the
following proposition:
\begin{proposition}[{\bfseries Upper-bound Complexity}]~\label{prop:upper}\newline
Let assume the same conditions as in Lemma~\ref{lemma:uniqpath}.
The number of linear segments in the regularization path of the Lasso is less than $(3^p+1)/2$.
\end{proposition}
\begin{proof}
We have already noticed that the number of linear segments of the path is at
most~$3^p$. Let us consider $\etab^\star(\lambda_1)\!\neq\!0$ for~$\lambda_1\!>\!0$. 
We now show that for all~$\lambda_2\!>\!0$, we have~$\etab^\star(\lambda_2)\!\neq\!-\etab^\star(\lambda_1)$,
and therefore the number of different sparsity patterns on the path~$\mathcal P$ is in fact less than or equal
to $(3^p\!+\!1)/2$. 

Let us assume that there exists~$\lambda_2\!>\!0$ with
$\etab^\star(\lambda_2)\!=\!-\etab^\star(\lambda_1)$, and look for a contradiction.
We define the set~$J'\defin\{ j \in \{1,\ldots,p\} : \etab^\star_j(\lambda_1) \neq 0\}$, 
and consider the solution of the reduced problem for all~$\lambda \geq 0$:
 \begin{displaymath}
    \tildew^\star(\lambda) \defin  \argmin_{\tildew \in \Real^{|J'|}} \frac{1}{2}\|\y-\X_{J'}\tildew\|_2^2 + \lambda \|\tildew\|_1,
\end{displaymath}
which is well defined since the optimization problem is strictly convex
(the conditions of Lemma~\ref{lemma:uniqpath} imply that~$\X_{J'}$ is
full rank). 
We remark that $\tildew^\star(\lambda_1)=\w^\star_{J'}(\lambda_1)$, and
$\tildew^\star(\lambda_2)=\w^\star_{J'}(\lambda_2)$.
Given the optimality conditions of Lemma~\ref{lemma:opt}, it is then easy to show that 
$\tildew^\star(0) = (\X_{J'}^\top\X_{J'})^{-1}\X_{J'}^\top \y = \frac{\lambda_2}{\lambda_1+\lambda_2}\tildew^\star(\lambda_1) + \frac{\lambda_1}{\lambda_1+\lambda_2}\tildew^\star(\lambda_2)$.
Since the signs of $\tildew^\star(\lambda_1)$ and~$\tildew^\star(\lambda_2)$ are opposite to each other and non-zero, we have $\|\tildew^\star(0)\|_1 \!<\! \|\tildew^\star(\lambda_1)\|_1$.
Independently, it is also~easy to show that the function~$\lambda \!\to\!
\|\tildew^\star(\lambda)\|_1$ should be non-increasing,
and we obtain a contradiction.
\end{proof}

In the next proposition, we present our adversarial strategy to build a
pathological regularization path. Given a Lasso problem with~$p$ variables and
a path~${\mathcal P}$, we design an additional
variable along with an extra dimension, such that the number of kinks of the
new path~$\tilde{\mathcal P}$ increases by a multiplicative factor compared to~$\mathcal P$.  We call our strategy
``adversarial'' since it consists of iteratively designing ``pathological''
variables.

\begin{proposition}[{\bfseries Adversarial Strategy}]~\label{prop:strategy}\newline
Let us consider $\y$ in~$\Real^n$ and~$\X$ in~$\Real^{n \times p}$ such that the
conditions of Lemma~\ref{lemma:uniqpath} are satisfied and~$\y$ is in the span of~$\X$. We denote by ${\mathcal P}$ the regularization path of the Lasso problem corresponding to~$(\y,\X)$, 
by~$k$ the number of linear segments of~${\mathcal P}$, and by~$\lambda_1 > 0$ the smallest value of the parameter~$\lambda$ corresponding to a kink
of~${\mathcal P}$. We define the vector~$\tildey$ in~$\Real^{n+1}$ and the matrix~$\tildeX$ in~$\Real^{(n+1) \times (p+1)}$ as follows:
\begin{displaymath}
   \tildey \defin \left[ \begin{array}{c} \y \\ y_{n+1} \end{array} \right],~~~~ \tildeX \defin \left[ \begin{array}{cc} \X & 2\alpha\y \\ 0 & \alpha y_{n+1} \end{array} \right],
\end{displaymath}
where $y_{n+1}\neq0$ and $0<\alpha<\lambda_1 / ({2\y^\top\y+y_{n+1}^2})$.

Then, the regularization path~$\tilde{\mathcal P}$
of the Lasso problem associated to $(\tildey,\tildeX)$ exists and has
$3k\!-\!1$ linear segments.
Moreover, let us consider~$\{\etab^1\!=\!0, \etab^2, \ldots, \etab^k \}$ the sequence of sparsity patterns in~$\{-1,0,1\}^p$ of~${\mathcal P}$ (the signs of the solutions $\w^\star(\lambda)$),
ordered from large to small values of~$\lambda$. The sequence of sparsity patterns in~$\{-1,0,1\}^{p+1}$ of the new path~${\tilde{\mathcal P}}$ is the following:
\begin{multline}
 \!\!\!\!\! 
 \Bigg\{ \overbrace{\left[\!\!\!       \begin{array}{c} \etab^1 \\ 0 \end{array} \!\!\!\right], 
           \left[\!\!\!       \begin{array}{c} \etab^2 \\ 0 \end{array}\!\!\!\right], \dots,
           \left[\!\!\!       \begin{array}{c} \etab^k \\ 0 \end{array}\!\!\!\right]}^{\text{first}~k~\text{patterns}}, 
        \overbrace{\left[\!\!\!       \begin{array}{c} \etab^k \\ 1 \end{array}\!\!\!\right],
           \left[\!\!\!       \begin{array}{c} \etab^{k\!-\!1} \\ 1 \end{array}\!\!\!\right],
           \ldots, \left[\!\!\!       \begin{array}{c} \etab^1\!=\!0 \\ 1 \end{array}\!\!\!\right]}^{\text{middle}~k~\text{patterns}}, \\
          \underbrace{\left[\!\!\!       \begin{array}{c} -\etab^2 \\ 1 \end{array}\!\!\!\right], 
           \left[\!\!\!       \begin{array}{c} -\etab^3 \\ 1 \end{array}\!\!\!\right], \ldots,
           \left[\!\!\!       \begin{array}{c} -\etab^k \\ 1 \end{array}\!\!\!\right]}_{\text{last}~k\!-\!1~\text{patterns}}  \Bigg\}.\!\!\!
   \label{eq:path}
\end{multline}
\end{proposition}
Let us first make some remarks about this proposition:\\
~\hspace*{0.5cm}$\bullet$ According to Eq.~(\ref{eq:path}) the sparsity
patterns of the new path~$\tilde{\mathcal P}$ are related to those of~$\mathcal P$.
More precisely, they have either the form
$[\etab^{i\top}, 0]^\top$ or~$[\pm \etab^{i\top}, 1]^\top$, 
where~$\etab^i$ is a sparsity pattern in~$\{-1,0,1\}^p$ of~$\mathcal P$. \\
~\hspace*{0.5cm}$\bullet$ The last column
of~$\tildeX$ involves a factor~$\alpha$ that controls its norm.  
With~$\alpha$ small enough, the~$(p\!+\!1)$-th variable enters late the
path~${\tilde{\mathcal P}}$. As shown in Eq.~(\ref{eq:path}), the first~$k$ sparsity patterns of
${\tilde{\mathcal P}}$ do not involve this variable and are exactly the same as those of~$\mathcal P$. \\
~\hspace*{0.5cm}$\bullet$ Let us give some intuition about the pathological behavior of the path~$\tilde{\mathcal P}$.
The first~$k$ kinks of~$\tilde{\mathcal P}$ are the same as those of~$\mathcal P$, and after
these first~$k$ kinks we have $\y\approx \X\w^\star(\lambda)$. Then,
the $(p\!+\!1)$-th variable enters the path and
we heuristically have
\begin{equation}
\tildeX \left[\!\! \begin{array}{c} \w^\star(\lambda) \\ 0 \end{array} \!\!\right] + \left[\!\!\begin{array}{c} 0 \\ y_{n+1} \end{array}\!\!\right] \approx \tildey \approx \tildeX \left[\!\! \begin{array}{c} -\w^\star(\lambda) \\ 1/\alpha \end{array} \!\!\right]. \label{eq:approx}
\end{equation}
The left side of Eq.~(\ref{eq:approx}) tells us that when
the~$(p\!+\!1)$-th variable is inactive, the coefficients associated to the first~$p$ variables
should be close to~$\w^\star(\lambda)$.
At the same time, the right side of Eq.~(\ref{eq:approx}) tells us that when the~$(p\!+\!1)$-th variable is
active, these same~$p$ coefficients should be instead close to~$-\w^\star(\lambda)$.
According to Eq.~(\ref{eq:path}), the signs of these $p$ coefficients along the path switch from~$\etab^k\!=\!\sign(\w^\star(\lambda))$ to~$-\etab^k$ by following the sequence $\etab^k,\etab^{k-1},\ldots,(\etab^1\!=\!0=\!-\etab^1),-\etab^2,\ldots,-\etab^k$, resulting in a path with~$3k\!-\!1$ linear segments.
The proof below more rigorously describes this strategy:
\begin{proof}
{\bfseries Existence of the new regularization path}: \newline
Let us rewrite the Lasso problem for~$(\tildey,\tildeX)$. 
\begin{multline}
\min_{\tildew \in \Real^p, {\tilde w} \in \Real}\frac{1}{2}\left\|\tildey-\tildeX \left[\begin{array}{c} \tildew \\ {\tilde w} \end{array} \right] \right\|_2^2 + \lambda \left\|\left[\begin{array}{c}\tildew \\ {\tilde w} \end{array}\right]\right\|_1, \\
=   \min_{{\tildew \in \Real^p, {\tilde w} \in \Real}} 
       \frac{1}{2}\|(1\!-\!2\alpha{\tilde w})\y\!-\!\X\tildew\|_2^2 \!+\! \frac{1}{2}(y_{n+1}\!-\!\alpha y_{n+1} {\tilde w})^2 
       \\ \!+\! \lambda \|\tildew\|_1 \!+\! \lambda|{\tilde w}|. 
\label{eq:lassosplit}
\end{multline}
Let~$(\tildew^\star,{\tilde w}^\star)$ be a solution for a given~$\lambda > 0$. By fixing~${\tilde w}={\tilde w}^\star$ in Eq.~(\ref{eq:lassosplit}) and optimizing with respect to~$\tildew$, we obtain an equivalent problem to~(\ref{eq:lassosplit}):
\begin{displaymath}
   \min_{\tildew' \in \Real^p} \frac{1}{2}\|\y-\X\tildew'\|_2^2 + \frac{\lambda}{|1-2\alpha {\tilde w}^\star|}\|\tildew'\|_1,
\end{displaymath}
with the change of variable~$\tildew = (1-2\alpha {\tilde w}^\star)\tildew'$ and assuming $1-2\alpha {\tilde w}^\star \neq 0$.
The solution of this problem is unique since it is a point of~${\mathcal P}$
and we therefore have
\begin{equation}
   \tildew^\star = \left\{ \begin{array}{ll} (1-2\alpha {\tilde w}^\star) \w^\star\Big(\frac{\lambda}{|1-2\alpha {\tilde w}^\star|}\Big) & ~~\text{if}~~ {\tilde w}^\star \neq \frac{1}{2\alpha} \\
    0 & ~~\text{otherwise}
   \end{array} \right.. \label{eq:closed2}
\end{equation}
Since the last column of~$\tildeX$ is not in the span of the first~$p$ columns by construction of~$\tildeX$,
it is then easy to see that the conditions of Lemma~\ref{lemma:uniqpath} are necessarily satisfied and therefore
$(\tildew^\star,{\tilde w}^\star)$ is in fact the unique solution of Eq.~(\ref{eq:lassosplit}).
Since this is true for all~$\lambda >0$, the regularization path is well defined, and we denote from now on
the above solutions by~$\tildew^\star(\lambda)$ and~${\tilde w}^\star(\lambda)$.

{\bfseries Maximum number of linear segments}: \newline
We now show that the number of linear segments of
the path is upper-bounded by~$3k\!-\!1$. Eq.~(\ref{eq:closed2}) shows that
$\sign({\tildew}^\star(\lambda))$ has the form~$\pm\etab^i$, where~$\etab^i$ in~$\{-1,0,1\}^p$ is one of the~$k$ sparsity
patterns from~${\mathcal P}$, whereas we have three
possibilities for~$\sign({\tilde w}^\star(\lambda))$, namely $\{-1,0,+1\}$.
Since one can not have two non-zero sparsity patterns that are opposite to each other
on the same path, as shown in the proof of
Proposition~\ref{prop:upper}, the number of possible sparsity
patterns reduces to~$3k\!-\!1$.

{\bfseries Characterization of the first $k$ linear segments}: \newline
Let us consider~$\lambda \!\geq\! \lambda_1$ and show that $\tildew^\star(\lambda)\!=\!\w^\star(\lambda)$ and~${\tilde w}^\star(\lambda)\!=\!0$ by 
checking the optimality conditions of Lemma~\ref{lemma:opt}.
The first $p$ equalities/inequalities in Eq.~(\ref{eq:opt}) are easy
to verify, the last one being also satisfied: 
\begin{displaymath}
|2\alpha\y^\top(\y\!-\!\X\w^\star(\lambda)) \!+\! \alpha y_{n+1}^2| 
 \leq2 \alpha \|\y\|_2^2 + \alpha y_{n+1}^2 < \lambda_1,
\end{displaymath}
where the last inequality is obtained from the definition of~$\alpha$.
Since this inequality is strict, this also ensures that there exists~$0 < \lambda'_1 < \lambda_1$ such that $\tildew^\star(\lambda)=\w^\star(\lambda)$ and~${\tilde w}^\star(\lambda)=0$ for all~$\lambda \geq \lambda_1'$. We have therefore shown that the first~$k$ sparsity patterns of the regularization path are given in Eq.~(\ref{eq:path}).

{\bfseries Characterization of the last~$2k\!-\!1$ segments}: \newline
We mainly use here the form of Eq.~(\ref{eq:closed2}) and
a few continuity arguments to characterize the rest of the path. First, we remark that
for all $\beta$ in $[0,\frac{1}{\alpha})$, there exists
a value for~$\lambda > 0$ such that~${\tilde w}^\star(\lambda)=\beta$. This is true because:
(i) $\lambda \! \to \! \tilde{w}^\star(\lambda)$ is continuous; (ii) $\tilde{w}^\star(\lambda_1) \!=\! 0$; (iii) $\tilde{w}^\star(0^+) \!=\! \frac{1}{\alpha}$.
Point (i) was shown in Lemma~\ref{lemma:uniqpath}, point (ii) in the previous paragraph, and point (iii) is necessary to have the term $(y_{n+1}\!-\!\alpha y_{n+1} {\tilde w})^2$ in Eq.~(\ref{eq:lassosplit}) go to~$0$ when~$\lambda$ goes to~$0^+$.

We now consider two values~$\lambda_1',\lambda_2'\!>\!0$ such that $\tilde{w}^\star(\lambda_1')\!=\!0$, $\tilde{w}^\star(\lambda_2')\!=\!\frac{1}{2\alpha}$ and $\tilde{w}^\star(\lambda) \in (0,\frac{1}{2\alpha})$ for all~$\lambda \in (\lambda_1',\lambda_2')$. On this open interval, we have that $(1-2\alpha{\tilde w}^\star(\lambda)) \! > \! 0$, and the continuous function $\lambda \to \lambda / |(1-2\alpha{\tilde w}^\star(\lambda))|$ ranges from $\lambda_1'$ to~$+\infty$. Combining this observation with Eq.~(\ref{eq:closed2}), we obtain that all sparsity patterns of the form~$[\etab^{i\top} , 1 ]^\top$ for~$i$ in~$\{1,\ldots,k\}$ appear on the regularization path.
With similar continuity arguments, it is easy to show that all sparsity patterns of the form~$[-\etab^{i\top} , 1 ]^\top$ for~$i$ in~$\{1,\ldots,k\}$ appear on the path as well.

We had previously identified~$k$ of the sparsity patterns, and now have
identified $2k\!-\!1$ different ones. Since we have at most~$3k\!-\!1$ linear segments, the set of sparsity patterns on the path~$\tilde{\mathcal P}$ is
entirely characterized. The fact that the sequence of sparsity patterns
is the one given in Eq.~(\ref{eq:path}) can easily be shown by reusing
similar continuity arguments.
\end{proof}
With this proposition in hand, we can now state the main result of this section:
\begin{theorem}[{\bfseries Worst-case Complexity}]~\newline
In the worst case, the regularization path of the Lasso has exactly $(3^p+1)/2$ linear segments. 
\end{theorem}
\begin{proof}
We start with~$n\!=\!p\!=\!1$, and define~$\y\!=\![1]$, and~$\X\!=\![1]$, leading to a
path with~$k=2$ segments. We then recursively apply
Proposition~\ref{prop:strategy}, keeping~$n\!=\!p$, choosing at iteration~$p+1$, $y_{p+1}\!=\!1$,
and a factor~$\alpha\!=\!\alpha_{p+1}$ satisfying the conditions of Proposition~\ref{prop:strategy}.
Denoting by~$k_p$ the number of linear segments at iteration~$p$, we have that~$k_{p+1}\!=\!3k_p\!-\!1$, 
and it is easy to show that~$k_p\!=\!(3^p\!+\!1)/2$. According to Proposition~\ref{prop:upper},
this is the longest possible regularization path.
Note that this example has a particularly simple shape:
\begin{displaymath}
   \y \defin \left[ \begin{array}{c} 1 \\ 1 \\ 1 \\ \vdots \\ 1  \end{array} \right],~~~~ \X \defin \left[ 
\begin{array}{ccccc} \alpha_1 & 2\alpha_2 & 2\alpha_3 & \ldots & 2\alpha_p \\ 
                    0 & \alpha_2 & 2\alpha_3 & \ldots & 2\alpha_p \\
                    0 & 0 & \alpha_3 & \ldots & 2\alpha_p \\
                    \vdots & \vdots & \vdots & \ddots & \vdots \\
                    0      &  0   &0   &  \ldots     & \alpha_p
\end{array} \right].
\end{displaymath}
\end{proof}
\subsection{Numerical Simulations}\label{sec:exp1}
We have implemented Algorithm~\ref{alg:homotopy} in Matlab, optimizing
numerical precision regardless of computational efficiency, which has allowed us
to check our theoretical results for small values of~$p$. For instance, we
obtain a path with $(3^p+1)/2=88\,574$ linear segments
for~$p=11$, and present such a pathological path in Figure~\ref{fig:path}. 
Note that when~$p$ gets larger, these examples quickly lead to
precision issues where some kinks are very close to each other. 
Our implementation and our pathological examples will be made publicly available.
In the next section, we
present more optimistic results on approximate regularization paths.
\begin{figure}
\centering
\includegraphics[width=0.75\linewidth]{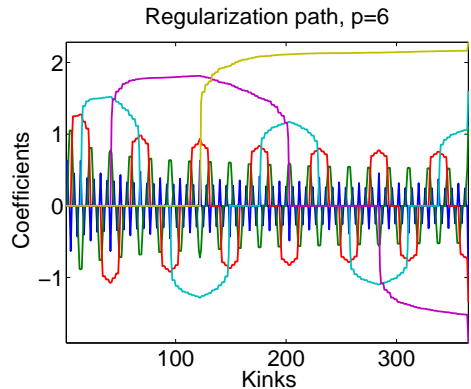}
\vspace*{-0.4cm}
\caption{Pathological regularization path with~$p\!=\!6$ variables and~$(3^6\!+\!1)/2\!=\!365$ kinks. The curves represent the values of the coefficients at every kink of the path. For visibility purposes, we use a non-linear scale and report the values $\sign(w)|w|^{0.1}$ for a coefficient~$w$. Best seen in color.}
\label{fig:path}
\vspace*{-0.3cm}
\end{figure}
\vspace*{-0.1cm}

\section{Approximate Homotopy}\label{sec:approx}
We now present another complexity analysis when exact solutions of
Eq.~(\ref{eq:lasso}) are not required. We follow in part the methodology
of~\citet{giesen2}, later refined by~\citet{giesen}, on approximate regularization paths of parameterized convex
functions.  Their results are quite general but, as we show later, we obtain
stronger results with an analysis tailored to the~Lasso.

A natural tool to guarantee the quality of approximate solutions is the duality
gap. Writing the Lagrangian of problem~(\ref{eq:lasso})
and minimizing with respect to the primal variable~$\w$ yields the following
dual formulation of~(\ref{eq:lasso}):
\begin{equation}
\vspace*{-0.1cm}
    \max_{\kappab \in \Real^n} -\frac{1}{2}\kappab^\top\kappab - \kappab^\top \y  \st \|\X^\top \kappab\|_\infty \leq \lambda, \label{eq:dual}
\vspace*{-0.05cm}
\end{equation}
where~$\kappab$ in~$\Real^n$ is a dual variable. Let us denote
by~$f_\lambda(\w)$ the objective function of the primal
problem~(\ref{eq:lasso}) and by~$g_\lambda(\kappab)$ the objective function of
the dual~(\ref{eq:dual}).  Given a pair of feasible primal and dual
variables~$(\w,\kappab)$, the difference~$\delta_\lambda(\w,\kappab) \defin
f_\lambda(\w)-g_\lambda(\kappab)$ is called a duality gap and provides an
optimality guarantee~\citep[see][]{borwein}:
\begin{displaymath}
   0 \leq f_\lambda(\w) - f_\lambda(\w^\star(\lambda))  \leq \delta_\lambda(\w,\kappab).
\end{displaymath}
In plain words, it upper bounds the difference between the current value of the
objective function~$f_\lambda(\w)$ and the optimal value of the objective function
$f_\lambda(\w^\star(\lambda))$.
In this paper, we use a relative duality gap criterion to guarantee
the quality of an approximate solution:\footnote{Note that our criterion is not exactly the same as in~\citet{giesen}. Whereas~\citet{giesen} consider a formulation where the~$\ell_1$-norm appears in a constraint, Eq.~(\ref{eq:lasso}) involves an~$\ell_1$-penalty. Even though these formulations have the same regularization path, they involve slightly different objective functions, dual formulations, and duality gaps. \label{ref:footnote1}}
\begin{definition}[{\bfseries $\varepsilon$-approximate Solution}]~\label{definition:opt1}\newline
let~$\varepsilon$ be in $[0,1]$. 
   A vector~$\w$ in~$\Real^p$ is said to be an~$\varepsilon$-approximate solution of problem~(\ref{eq:lasso}) if  
   there exists~$\kappab$ in~$\Real^n$ such that $\|\X^\top \kappab\|_\infty\! \leq\! \lambda$ and
      $\delta_\lambda(\w,\kappab)\! \leq\! \varepsilon f_\lambda(\w)$. 

   Given a set~$\tilde{\mathcal P} \defin \{\tildew(\lambda) \in \Real^p : \lambda > 0\}$, we say that~$\tilde{\mathcal P}$ is an~$\varepsilon$-approximate regularization path if any point~$\tildew(\lambda)$ of~$\tilde{\mathcal P}$ is an~$\varepsilon$-approximate solution for problem~(\ref{eq:lasso}).
\end{definition}
Our goal is now to build~$\varepsilon$-approximate regularization
paths and study their complexity.
To that effect, we introduce approximate optimality conditions
based on small perturbations of those given in
Lemma~\ref{lemma:opt}:
\begin{definition}[{\bfseries $OPT_\lambda(\varepsilon_1,\varepsilon_2)$ Condition}]~\label{definition:opt2}\newline
Let~$\varepsilon_1\! \geq\! 0$ and~$\varepsilon_2\! \geq\! -\varepsilon_1$. 
A vector~$\w$ in~$\Real^p$ satisfies the $OPT_\lambda(\varepsilon_1,\varepsilon_2)$ condition if and only if
 for all $1\!\leq\! j\!\leq\! p$, 
\begin{equation}
   \begin{split}
   \lambda(1\!-\!\varepsilon_2)\! \leq\! \x^{j\top}(\y\!-\!\X\w)\sign(\w_j)&\!\leq \!\lambda(1\!+\!\varepsilon_1) ~\text{if}~ \w_j\! \neq\! 0, \\
   |\x^{j\top}(\y\!-\!\X\w)|  &\!\leq\!\lambda(1\!+\!\varepsilon_1) ~\text{otherwise}.
   \end{split} \label{eq:optapprox}
\end{equation}
\end{definition}
\vspace*{-0.2cm}
Note that when~$\varepsilon_1\!=\!\varepsilon_2\!=\!0$, this condition reduces to the
exact optimality conditions of Lemma~\ref{lemma:opt}. Of interest for us is the
relation between Definitions~\ref{definition:opt1} and~\ref{definition:opt2}.
Let us consider a vector~$\w$ such that
$OPT_\lambda(\varepsilon_1,\varepsilon_2)$ is satisfied.
Then, the vector~$\kappab \defin \frac{1}{1+\varepsilon_1}(\X\w-\y)$ is feasible 
for the dual~(\ref{eq:dual}) and we can compute a duality~gap:
\begin{displaymath}
\vspace*{-0.1cm}
   \begin{split}
      \delta_{\lambda}(\w,\kappab) & = f_\lambda(\w) - g_\lambda(\kappab) \\
                                   & = \frac{1}{2}(1+\varepsilon_1)^2\kappab^\top\kappab + \lambda\|\w\|_1 + \frac{1}{2}\kappab^\top\kappab + \kappab^\top\y \\
                                   & = \frac{\varepsilon_1^2}{2}\kappab^\top\kappab + \lambda\|\w\|_1 + \kappab^\top\Big(\y+(1+\varepsilon_1)\kappab\Big) \\
                                   & = \frac{\varepsilon_1^2}{(1+\varepsilon_1)^2}\frac{1}{2}\|\y-\X\w\|_2^2 + \lambda\|\w\|_1 + \kappab^\top\X\w.
   \end{split}
\vspace*{-0.1cm}
\end{displaymath}
From Eq.~(\ref{eq:optapprox}), it is easy to show that~$\lambda\|\w\|_1+\kappab^\top\X\w \leq \frac{\varepsilon_1+\varepsilon_2}{1+\varepsilon_1}\lambda\|\w\|_1$, and we can obtain the following bound:
\begin{equation}
    \delta_{\lambda}(\w,\kappab) \leq \max\bigg( \frac{\varepsilon_1^2}{(1+\varepsilon_1)^2} , \frac{\varepsilon_1+\varepsilon_2}{1+\varepsilon_1} \bigg)f_\lambda(\w). \label{eq:dualitygap}
\end{equation}
From this upper bound, we derive our first result:
\begin{proposition}[{\bfseries Approximate Analysis}]~\label{prop:approx1}\newline
Let~$\y$ be in~$\Real^n$ and~$\X$ in~$\Real^{n \times p}$ such that the conditions of Lemma~\ref{lemma:uniqpath} are satisfied.
Let~$\lambda_{\infty}\!\defin\! \|\X^\top\y\|_\infty$ be the value of~$\lambda$ corresponding to the start of the path, and~$\lambda_1 \!>\! 0$ be the one corresponding to the last kink.
For all~$\varepsilon\!\in\!(0,1)$, there exists an $\varepsilon$-approximate regularization path with at most $\Big\lceil \frac{\log(\lambda_\infty / \lambda_1)}{\sqrt{\varepsilon}} \Big\rceil$ linear segments.
\end{proposition}
\vspace*{-0.4cm}
\begin{proof}
From Eq.~(\ref{eq:optapprox}), one can show by a simple calculation that an exact
solution $\w^\star(\lambda)$ for a given~$\lambda$ satisfies
$OPT_{\lambda(1\!-\!\varepsilon_3)}(\varepsilon_3/(1\!-\!\varepsilon_3),-\varepsilon_3/(1\!-\!\varepsilon_3))$.
According to Eq.~(\ref{eq:dualitygap}), there exists a dual
variable~$\kappab$ such that~$\delta_{\lambda(1-\varepsilon_3)}(\w^\star(\lambda),\kappab) \!\leq\!
\varepsilon_3^2$. Thus, for any~$\lambda'$ chosen
in~$[\lambda,\lambda(1\!-\!\sqrt{\varepsilon})]$, the solution~$\w^\star(\lambda)$
is an $\varepsilon$-approximate solution for the parameter~$\lambda'$.
Between~$\lambda_\infty$ and~$\lambda_1$, we can obtain an~$\varepsilon$-approximate piecewise linear (in fact piecewise constant) regularization path
by sampling solutions~$\w^\star(\lambda)$ for~$\lambda$ in $\{\lambda_\infty, \lambda_\infty(1\!-\!\sqrt{\varepsilon}),\ldots,\lambda_\infty(1\!-\!\sqrt{\varepsilon})^k,\lambda_1\}$ with $\lambda_\infty(1\!-\!\sqrt{\varepsilon})^{k+1} \! \leq \! \lambda_1$. The number of segments of the corresponding approximate path is at most
$\Big\lfloor \frac{-\log(\lambda_\infty / \lambda_1)}{\log(1-\sqrt{\varepsilon})} \Big\rfloor \!+\!1 \!\leq\! \Big\lceil \frac{\log(\lambda_\infty / \lambda_1)}{\sqrt{\varepsilon}} \Big\rceil$.
\end{proof}
\vspace*{-0.1cm}
Note that the term~$\lambda_\infty / \lambda_1$ is possibly large, but it
is controlled by a logarithmic function and can be considered as constant
for finite precision machines. In other words, 
the complexity of the approximate path is upper-bounded
by~$O(1/\sqrt{\varepsilon})$.
In contrast, the analysis of~\citet{giesen2} and \citet{giesen} give us:\\
~\hspace*{0.5cm}$\bullet$~an approximate path with $O(1/\varepsilon)$ linear segments can be obtained with a weaker
approximation guarantee than ours. Namely, a bound~$\delta \!\leq\! \varepsilon$ along the path, where~$\delta$ is a 
duality gap, whereas we use \emph{relative}
duality gaps of the form~$\delta \!\leq\! \varepsilon
f_\lambda(\w)$;\footnote{When there exists~$m,M \!>\! 0$ such that~$m\!<\!f_\lambda\!<\!M$, the relative duality gap guarantee is
similar (up to a constant) to the simple bound~$\delta \leq \varepsilon$. However, we have for the Lasso that~$f_\lambda(\w^\star(\lambda)) \to 0$ when~$\lambda$ goes
to~$0^+$, as long as~$\y$ is in the span of~$\X$. Note that as noticed in footnote~\ref{ref:footnote1},~\citet{giesen} uses a slightly different duality gap than ours.}
Interestingly, this bound is proven to be optimal in the context of parameterized convex functions on the~$\ell_1$-ball. Our result show that such bound can be improved for the Lasso. \\
~\hspace*{0.5cm}$\bullet$~$\!\!$a methodology to obtain relative duality gaps along the path, which can easily 
provide complexity bounds for the full path of different problems, notably support vector machines, but not for the Lasso.

Proposition~\ref{prop:approx1} is optimistic, but not practical since
it requires sampling \emph{exact} solutions of the path~$\mathcal P$.  We
introduce an approximate homotopy method in Algorithm~\ref{alg:homotopy2} which
does not require computing exact solutions and still enjoys a similar complexity.
It exploits the piecewise linearity of the path, but uses a first-order method~\citep{beck,fu} when the linear segments
of the path are too short.
\vspace*{-0.1cm}
\begin{algorithm}[ht!]
\caption{Approximate Homotopy for the Lasso.}\label{alg:homotopy2}
\begin{algorithmic}[1]
\INPUT a vector~$\y$ in $\Real^n$, a matrix~$\X$ in~$\Real^{n \times p}$, the required precision~$\varepsilon \in [0,1]$; $\lambda_{1} > 0$;
\STATE {\bfseries initialization:} set $\lambda$ to $\|\X^\top\y\|_\infty$; set~$\tildew(\lambda)=0$;
\STATE set $\theta = 1+\varepsilon/2 - \sqrt{\varepsilon}/2$;
\STATE set $J\defin \{ j_0 \}$ such that~$|\x^{j_0\top}\y|=\lambda$;
\WHILE{$\lambda \geq \lambda_{1}$} 
\STATE {\bfseries if} $(\X_J^\top\X_J)$ is not invertible {\bfseries then} go to \ref{step:skip};\label{item:lars2}
\STATE set~$\tildeetab \defin (1/\lambda)\X^\top(\y\!-\!\X\tildew(\lambda))$;
\STATE \label{item:step2} compute the approximate direction of the path:
\vspace*{-0.2cm}
\begin{displaymath}
\left\{ \begin{array}{rcl}
\tildew_J(\lambda) &\!=\!& (\X_J^\top\X_J)^{-1}(\X_J^\top\y\!-\!\lambda \tildeetab_J) \\
\tildew_{J^\complement}(\lambda)&\!=\!&0.
\end{array}\right.
\vspace*{-0.1cm}
\end{displaymath}
Find the smallest step~$\tau > 0$ such that: \\
       $\bullet$~there exists $j$ in $J^\complement$ such that \\ 
       $|\x^{j\top}(\y\!-\!\X\tildew(\lambda\!-\!\tau))|\! =\!(\lambda\!-\!\tau)(1\!+\!\frac{\varepsilon}{2})$; add $j$ to~$J$; \\
       $\bullet$~there exists $j$ in $J$ such that~$\tildew_j(\lambda)\!\neq\! 0$ and $\tildew_j(\lambda\!-\!\tau)\!=\!0$; remove $j$ from $J$; \\
\IF{$\tau \geq \lambda\theta\sqrt{\varepsilon}$}
   \STATE replace~$\lambda$ by~$\lambda-\tau$;
\ELSE
   \STATE replace~$\lambda$ by~$\lambda(1-\theta\sqrt{\varepsilon})$; \label{step:skip}
   \STATE use a first-order optimization method to find a solution~$\tildew(\lambda)$ satisfying $OPT_\lambda(\varepsilon/2,\varepsilon/2)$; \label{item:step3}
   \STATE set~$J=\{ j \in \{1,\ldots,p\} : \tildew_j(\lambda) \neq 0 \}$.
\ENDIF
\STATE record the pair $(\lambda,\tildew(\lambda))$;
\ENDWHILE
\STATE {\bf{Return:}} sequence of recorded values~$(\lambda,\tildew(\lambda))$.
\end{algorithmic}
\end{algorithm}

\vspace*{-0.8cm}
Note that when~$\varepsilon\!=\!0$, Algorithm~\ref{alg:homotopy2} reduces to
Algorithm~\ref{alg:homotopy}.
Our approach exploits the following ideas, which we formally prove in the sequel. Assume that~$\tildew(\lambda)$ satisfies~$OPT_\lambda(\varepsilon/2,\varepsilon/2)$. Then, \\
~\hspace*{0.5cm}$\bullet$ $\tildew(\lambda)$ is an~$\varepsilon$-approximation for
   all~$\lambda'$ in~$[\lambda,\lambda(1-\theta\sqrt{\varepsilon})]$. This
   guarantees us that one can always make step sizes for~$\lambda$
   greater than or equal to~$\lambda\theta\sqrt{\varepsilon}$; \\
 ~\hspace*{0.5cm}$\bullet$ the direction followed in Step~\ref{item:step2} maintains 
 $OPT_\lambda(\varepsilon/2,\varepsilon/2)$, but when two kinks are too close to each other---that is, $\tau \!<\!
   \lambda\theta\sqrt{\varepsilon}$, we directly look for
   a solution for the parameter~$\lambda'\!=\!\lambda(1\!-\!\theta\sqrt{\varepsilon})$
   that satisfies $OPT_{\lambda'}(\varepsilon/2,\varepsilon/2)$.
   Any first-order method can be used for that purpose, e.g., a proximal
   gradient method~\citep{beck}, using the current value~$\tildew(\lambda)$ as
   a warm start. \\
Note also that when $(\X_J^\top\X_J)$ is not invertible, the method uses first-order steps.
The next proposition precisely describes the guarantees of our algorithm.
\begin{proposition}[{\bfseries Analysis of Algorithm~\ref{alg:homotopy2}}]~\label{prop:approx2}\newline
Let~$\y$ be in~$\Real^n$ and~$\X$ in~$\Real^{n \times p}$. 
 For all~$\lambda_1\!>\!0$ and~$\varepsilon\! \in \!(0,1)$, Algorithm~\ref{alg:homotopy2} returns
 an~$\varepsilon$-approximate regularization path
 on~$[\lambda_\infty,\lambda_1]$.  Moreover, it terminates in at most
 $\Big\lceil \frac{\log(\lambda_\infty /
 \lambda_{1})}{\theta\sqrt{\varepsilon}} \Big\rceil$ iterations,
 where~$\lambda_{\infty}\defin \|\X^\top\y\|_\infty$.  
 \end{proposition}
 \begin{proof}
 We first show that any solution on the path is an~$\varepsilon$-approximate solution.
 First, it is easy to check that~$OPT_\lambda(\varepsilon/2,\varepsilon/2)$
 is always satisfied at Step~\ref{item:lars2}. This
 is either a consequence of Step~\ref{item:step3}, or because the direction
 $\tildew_J(\lambda') = (\X_J^\top\X_J)^{-1}(\X_J^\top\y\!-\!\lambda'
 \tildeetab_J)$ maintains $OPT_{\lambda'}(\varepsilon/2,\varepsilon/2)$ when
 $\lambda'$ varies between~$\lambda$ and~\mbox{$\lambda\!-\!\tau$}. 
 From Eq.~(\ref{eq:dualitygap}), we obtain that~$\tildew(\lambda)$ is an
 $\varepsilon$-approximate solution whenever
 $OPT_\lambda(\varepsilon/2,\varepsilon/2)$ is satisfied.
 Thus, we only need to check that $\tildew(\lambda)$ is also an~$\varepsilon$-approximate solution for~$\lambda'$ in~$[\lambda,\lambda(1-\theta\sqrt{\varepsilon})]$:
 for~$\varepsilon_3 \geq 0$, it is easy to check that $OPT_\lambda(\varepsilon/2,\varepsilon/2)$ implies $OPT_{\lambda(1\!-\!\varepsilon_3)}((\varepsilon/2\!+\!\varepsilon_3)/(1\!-\!\varepsilon_3),(\varepsilon/2\!-\!\varepsilon_3)/(1\!-\!\varepsilon_3)).$  Setting~$\varepsilon_3\!=\!\theta\sqrt{\varepsilon}$ and using Eq.~(\ref{eq:dualitygap}), it is possible to show that the desired condition is satisfied.

 Since the step size for~$\lambda$ is always greater
 than~$\lambda\theta\sqrt{\varepsilon}$, the maximum number of
 iterations is upper-bounded by $\Big\lfloor \frac{-\log(\lambda_\infty / \lambda_1)}{\log(1-\theta\sqrt{\varepsilon})} \Big\rfloor+1 \leq \Big\lceil \frac{\log(\lambda_\infty / \lambda_{1)}}{\theta\sqrt{\varepsilon}} \Big\rceil$
 \end{proof}
 We remark that the scalar~$\theta$ is very close to~$1$ and therefore the
 complexity is similar to the one of Proposition~\ref{prop:approx1}, with a
 logarithmic function controlling the possibly large term~$\lambda_\infty /
 \lambda_{1}$.  This algorithm is practical in different aspects: (i)
 it is almost as simple to implement as the homotopy method; (ii) it is robust
 to cases where two kinks are too close for the classical homotopy method to work;
 (iii) it provides optimality guarantees along the path; (iv) whenever possible, 
 it explicitly exploits the piecewise linearity of the path.
 We next present experiments to verify our analysis.

\subsection{Numerical Simulations}
We have implemented Algorithm~\ref{alg:homotopy2} with a few modifications to
the code used in Section~\ref{sec:exp1}. The inner solver is 
a coordinate descent algorithm~\citep[see][]{fu}, with a stopping
criterion based on Definition~\ref{definition:opt2}.

We consider~$4$ datasets. The first one dubbed \textsf{SYNTH}
consists of a pure noise fitting
scenario with no statistical meaning.  The entries of the corresponding vector~$\y$ and
matrix~$\X$ are i.i.d. draws from a standard normal distribution. The next dataset
is called~\textsf{PATHOL} and is a pathological example
obtained from the analysis of Section~\ref{sec:worst}. Finally, we consider two
datasets based on real data, respectively dubbed~\textsf{MADELON}\footnote{\url{http://www.nipsfsc.ecs.soton.ac.uk/datasets/}.} and~\textsf{PCMAC}\footnote{\url{http://featureselection.asu.edu/datasets.php}.}.
For each dataset, we center and normalize the columns of~$\X$ and the vector~$\y$, and
choose the parameter~$\lambda_1$ corresponding to the last kink of the true~path.

\begin{table}[t]
\caption{Complexity results of $\varepsilon$-approximated regularization paths for four datasets with~$n$ observations and~$p$ variables. The number of linear segments is denoted by~$k$.} \label{table:complex}
\centering
\vspace*{0.1cm}
{\footnotesize
\begin{tabular}{|l||c|c|c|c|}
\hline
   & {\scriptsize \textsf{SYNTH}} & {\scriptsize \textsf{PATHOL}} & {\scriptsize \textsf{MADELON}} & {\scriptsize \textsf{PCMAC}} \\
\hline
$n$ & $1\,100$        & 11             & $2\,000$    & $1\,943$     \\   
\hline
$p$ & $1\,000$             & 11             & 500         & $3\,289$     \\   
\hline
\hline
$k$, full path & $1\,615$ & $88\,574$  & $517$  & $2\,561$    \\   
\hline
\hline
$k$, $\varepsilon\!=\!10^{-5}$ & $1\,297$ & $2\,744$ & $468$ & $1\,254$ \\   
\hline
$k$, $\varepsilon\!=\!10^{-4}$ & $686$ & $1\,071$ & $327$ & $444$ \\   
\hline
$k$, $\varepsilon\!=\!10^{-3}$ & $268$ & $405$ & $152$ & $155$ \\   
\hline
$k$, $\varepsilon\!=\!10^{-2}$ & $96$ & $146$ & $61$ & $53$ \\   
\hline
$k$, $\varepsilon\!=\!0.1$ & $34$ & $51$ & $22$ & $18$ \\   
\hline
$k$, $\varepsilon\!=\!0.25$ & $21$ & $32$ & $15$ & $11$ \\   
\hline
$k$, $\varepsilon\!=\!0.5$ & $14$ & $20$ & $10$ & $7$ \\   
\hline
\end{tabular}
}
\vspace*{-0.4cm}
\end{table}

For all datasets, we compute the full regularization path using
Algorithm~\ref{alg:homotopy} and several~$\varepsilon$-approximate regularization
paths using Algorithm~\ref{alg:homotopy2}.
Note that the path of~\textsf{PCMAC} was stopped around~$\lambda \!\approx\!
10^{-4}$ where the matrix~$\X_J^\top\X_J$ became ill-conditioned and the Lasso solution dense.
As a simple sanity check, we first experimentally verify the correctness of
Propositions~\ref{prop:approx1} and~\ref{prop:approx2}, by sampling
solutions on the approximate path we obtain, computing duality gaps, and checking that the solutions are indeed
$\varepsilon$-approximate.  We conclude that our experimental results match our
theoretical analysis. We present the different path complexities in
Table~\ref{table:complex}.

Interestingly, the complexity of the pathological example significantly reduces
when one is looking for an approximate solution. For example, for~$\varepsilon\!=\!10^{-3}$, the complexity of the
approximate path is less than~$0.5\%$ the one of the full path.  This significantly
contrasts with the pessimistic result obtained in Section~\ref{sec:worst}.
As expected, the two examples based on real data exhibit a path complexity
of the same order of the problem size, which also significantly reduces
when~$\varepsilon$ increases.

\section{Conclusion}\label{sec:ccl}
We have presented new results on the regularization path and thus on
homotopy methods for the Lasso. First, we have shown that the path has an exponential 
worst-case complexity, which, as far as we know, had never been formally proved before. 
Our second result is more optimistic, and shows that when an exact path
is not required, only a relatively small number of points on the path
need to be computed. Finally, we propose a practical approximate homotopy algorithm,
which can provide such approximate paths at a desired precision.


\section*{Acknowledgments}
This paper was supported in part by NSF grants SES-0835531, CCF-0939370,
DMS-1107000, DMS-0907632, and by ARO-W911NF-11-1-0114.

{
\bibliography{main}
\bibliographystyle{icml2012}
}

\end{document}